\documentclass{article} 
\usepackage{iclr2016_conference,times}
\usepackage{defs_151115}
\usepackage{hyperref}
\usepackage{url}
\usepackage{algorithm}
\usepackage[noend]{algpseudocode}
\usepackage{bbm}
\usepackage{wrapfig}
\usepackage{bm}
\usepackage{color}
\usepackage{graphicx}
\usepackage{amsthm}
\usepackage{amsmath}
\usepackage{amssymb}
\usepackage{caption}
\usepackage{subcaption}
\usepackage{adjustbox}
\usepackage{enumitem}
\usepackage{lipsum}
\usepackage{multirow}
\setlist{nolistsep}

\DeclareMathOperator*{\argmax}{argmax}
\DeclareMathOperator*{\argmin}{argmin}

\newtheorem{lemma}{Lemma}
\newtheorem{theorem}{Theorem}
\newtheorem{proposition}{Proposition}

\newtheorem{example}{Example}

\graphicspath{ {./figs/} }

\let\svthefootnote\thefootnote

\title{Denoising Criterion for Variational Auto-encoding Framework}

\author{Daniel Jiwoong Im, Sungjin Ahn, Roland Memisevic, Yoshua Bengio$^\dagger$\\
Montreal Institute for Learning Algorithms\\
University of Montreal\\
Montreal, QC, H3C 3J7\\
\texttt{\{imdaniel,ahnsungj,memisevr,<findme>\}@iro.umontreal.ca}\\
}

\begin{document}

\maketitle

\let\thefootnote\relax\footnote{$^\dagger$CIFAR Senior Fellow}
\addtocounter{footnote}{-1}\let\thefootnote\svthefootnote

\begin{abstract}
Denoising autoencoders (DAE) are trained to reconstruct their clean inputs
with noise injected at the input level, while variational autoencoders
(VAE) are trained with noise injected in their stochastic hidden layer,
with a regularizer that encourages this noise injection. In this paper,
we show that injecting noise both in input and in the stochastic hidden
layer can be advantageous and we propose a modified variational lower
bound as an improved objective function in this setup. When input is corrupted,
then the standard VAE lower bound involves marginalizing the encoder conditional
distribution over the input noise, which makes the training criterion intractable.
Instead, we propose a modified training criterion which corresponds to
a tractable bound when input is corrupted.
Experimentally, we find that the proposed
denoising variational autoencoder (DVAE) yields better average log-likelihood than the
VAE and the importance weighted autoencoder on the MNIST and Frey Face datasets.
\end{abstract}

\section{Introduction}
Variational inference \citep{Jordan1999} has been a core component of approximate Bayesian inference along with the Markov chain Monte Carlo (MCMC) method \citep{Neal93}. It has been popular to many researchers and practitioners because the problem of learning an intractable posterior distribution is formulated as an optimization problem which has many advantages compared to MCMC; (i) we can easily take advantage of many advanced optimization tools \citep{kingma2014adam, duchi2011adaptive,zeiler2012adadelta}, (ii) the training by optimization is usually faster than the MCMC sampling, and (iii) unlike MCMC where it is difficult to decide when to finish the sampling, the stopping criterion in variational inference is clear. 

One remarkable recent advance in variational inference is to use the inference network (also known as the recognition network) as the approximate posterior distribution \citep{Kingma2014vae,Rezende2014,dayan1995helmholtz,bornschein2014reweighted}. Unlike the traditional variational inference where  different variational parameters are required for each latent variable, in the inference network, the approximate posterior distribution for each latent variable is conditioned on an observation and the parameters are shared among the latent variables. Combined with advances in training techniques such as the re-parameterization trick and the REINFORCE \citep{Williams1992,Mnih2014}, it became possible to train  variational inference models efficiently for large-scale datasets.

Despite these advances, it is still a major problem to obtain a class of variational distributions which is flexible enough to accurately model the true posterior distribution. For instance, in the variational autoencoder (VAE), in order to achieve efficient training, each dimension of the latent variable is assumed to be independent each other (i.e., there are factorized) and modeled by a univariate Gaussian distribution whose parameters (i.e., the mean and the variance) are obtained by a nonlinear projection of the input using a neural network. Even when VAE performs well in practice for a rather simple problems such as generating small and simple images (e.g., MNIST), it is still required to relax this strong restriction on the variational distributions in order to apply it to more complex real-world problems. Recently, there have been efforts in this direction. \citet{Salimans2015} integrated MCMC steps into the variational inference such that the variational distribution becomes closer to the target distribution as it takes more MCMC steps inside each iteration of the variational inference. Similar ideas but applying a sequence of invertible non-linear transformations rather than MCMC are also proposed by \citep{Dinh2015} and \citep{Rezende2015}.

On the other hand, the denoising criterion, where the input is corrupted by adding some noise and the model is asked to recover the original input, has been studied extensively for deterministic generative models \citep{Seung1998,Vincent2008,bengio2013generalized}. The study showed that the denoising criterion plays an important role in achieving good generalization performance \citep{Vincent2008} because it makes the nearby data points in the low dimensional manifold to be robust against the presence of small noise in the high dimensional observation space  \citep{Seung1998, Vincent2008, Rifai2011, Guillaume2014, Im2016}. Therefore, it is natural to ask if the denoising criterion (where we add the noise to the inputs) can also be advantageous for the variational auto-encoding framework where the noise is added to the latent variables, not the inputs, and if so, how can we formulate the problem for efficient training. Although it has not been considerably studied how to combine these, there has been some evidences of its usefulness\footnote{In practice, it turned out to be useful to \textit{augment} the dataset by adding some random noise to the inputs. However, in denoising criterion, unlike the augmenting, the model tries to recover the original data, not the corrupted one.}. For example, \citet{Rezende2014} pointed out that injecting additional noise to the recognition model is crucial to achieve the reported accuracy for unseen data, advocating that in practice denoising can help the regularization of probabilistic generative models as well.

In this paper, motivated by the DAE and the VAE, we study the denoising criterion for variational inference based on recognition networks, which we call the \textit{variational auto-encoding framework} throughout. 
Our main contributions are as follows. We introduce a new class of approximate distributions where the recognition network is obtained by marginalizing the input noise over a corruption distribution, 
and thus provides capacity to obtain a more flexible approximate distribution class such as the mixture of Gaussian. 
Because applying this approximate distribution to the standard VAE objective makes the training intractable, we propose a new objective, 
called the denoising variational lower bound, and show that, given a sensible corruption function, this is 
(i) tractable and efficient to train
, and (ii) easily applicable to many existing models such as the variational autoencoder, the importance reweighted autoencoder (IWAE) \citep{Burda2015}, and the neural variational inference and learning (NVIL) \citep{Mnih2014}.
In the experiments, we empirically demonstrate that the proposed denoising criterion for variational auto-encoding framework helps to improve the performance in both the variational autoencoders and the importance weighted autoencoders (IWAE) on the binarized MNIST dataset and the Frey Face dataset. 

\section{Background}
 
Variational inference is an approximate inference method where the goal is to approximate the intractable posterior distribution $p(\bz|\bx)$, by a tractable approximate distribution $q_\phi(\bz)$. Here, $\bx$ is the observation and $\bz \in \eR^D$ is the model parameters or latent variables. To keep it tractable, the approximate distributions are limited to a restricted family of distributions $q_\phi \in \cQ$ parameterized by variational parameters $\phi$. For example, in the mean-field variational inference \citep{Jordan1999}, the distributions in $\cQ$ treat all dependent variables as independent, i.e., $q(\bz) = \prod q_d(z_d)$. 

The basic idea of obtaining the optimal approximate distribution $q_{\phi^*} \in \cQ$ is to find the variational parameter $\phi^*$ that minimizes the Kullback-Leibler (KL) divergence between the approximate distribution $q_{\phi}$ and the target distribution $p$. Although the KL divergence itself involves the intractable target distribution, instead of directly minimizing the KL divergence, we can bypass it by decomposing the marginal log-likelihood as follows:
\bea
\log p(\bx) = \eE_{q_\phi(\bz)}\left[ \log \f{p(\bx, \bz)}{q_\phi(\bz)}\right] + \KL(q_\phi(\bz) || p(\bz|\bx)).
\eea
That is, observing that the marginal log-likelihood $\log p(\bx)$ is independent of the variational distribution $q_\phi$ and the KL term is non-negative, instead of minimizing the KL divergence, we can maximize the first term, called the variational lower bound, which is the same as minimizing the KL divergence term. Thus, in variational inference, we transform the problem of learning a distribution to an optimization problem of maximizing the variational lower bound with respect to the variational parameter $\phi$.

\subsection{Variational AutoEncoders}

The variational autoencoder (VAE) \citep{Kingma2014vae,Rezende2014} is a particular
type of variational inference framework which is closely related to our
focus in this work. With the VAE, the posterior distribution is defined as
$p_\ta(\bz|\bx) \propto p_\ta(\bx|\bz) p(\bz)$. Specifically, we define a
prior $p(\bz)$ on the latent variable $\bz \in \eR^D$, which is usually set
to an isotropic Gaussian distribution $\cN(0, \sig\eI_D)$. Then, we use a
parameterized distribution to define the observation model
$p_\ta(\bx|\bz)$. A typical choice for the parameterized distribution is to
use a neural network where the input is $\bz$ and output a parametric
distribution over $\bx$, such as the Gaussian or Bernoulli,
depending on the data type.
Then, $\ta$ becomes the weights of the neural network. We call
this network $p_\ta(\bx|\bz)$ the \textit{generative network}. Due to the
complex nonlinearity of the neural network, the posterior distribution
$p_\ta(\bz|\bx)$ is intractable.

One interesting aspect of VAE is that the approximate distribution $q$ is
conditioned on the observation $\bx$, resulting in a form
$q_\phi(\bz|\bx)$. Similar to the generative network, we use a neural
network for $q_\phi(\bz|\bx)$ with $\bx$ and $\bz$ as its input and output,
respectively. The variational parameter $\phi$, which is also the weights
of the neural network, is shared among all observations. We call this
network $q_\phi(\bz|\bx)$ the \textit{inference network},
\textit{recognition network}.

The objective of VAE is to maximize the following variational lower bound with respect to the parameters $\ta$ and $\phi$.
\bea
\log p_\ta(\bx) &\ge& \eE_{q_\phi(\bz|\bx)}\left[ \log \f{p_\ta(\bx, \bz)}{q_\phi(\bz|\bx)}\right]\\ 
\eqa \eE_{q_\phi(\bz|\bx)}\left[ \log p_\ta(\bx|\bz)\right] - \KL(q_\phi(\bz|\bx)||p(\bz))\label{eqn:vae}.
\eea

Note that in Eqn. \eqref{eqn:vae}, we can interpret the first term as a
reconstruction accuracy through an autoencoder with noise injected in
the hidden layer that is the output of the inference network, 
and the second term as a regularizer which enforces
the approximate posterior to be close to the prior and maximizes
the entropy of the injected noise.

The earlier approaches to train this type of model were based on the
variational EM algorithm: in the E-step, fixing $\ta$, we update $\phi$
such that the approximate distribution $q_\phi(\bz|\bx)$ close to the true
posterior distribution $p_\ta(\bz|\bx)$, and then in the M-step, fixing
$\phi$, we update $\ta$ to increase the marginal log-likelihood. However,
with the VAE it is possible to apply the backpropagation on the variational
parameter $\phi$ by using the re-parameterization trick
\citep{Kingma2014vae}, considering $\bz$ as a function of i.i.d. noise
and of the output of the encoder (such as the mean and variance of the
Gaussian). Armed with the gradient on these parameters, the gradient on the
generative network parameters $\ta$ can readily be computed by back-propagation, and thus we can
jointly update both $\phi$ and $\ta$ using efficient optimization
algorithms such as the stochastic gradient descent.

Although our exposition in the following proceeds mainly with the VAE model
for simplicity, the proposed method can be applied to a more general class
of variational inference methods which use the inference network 
$q_\phi(\bz|\bx)$ for the approximate distribution. This includes other
recent models such as the importance weighted autoencoders (IWAE), the
neural variational inference and learning (NVIL), and DRAW
\citep{Gregor2015}.

\section{Denoising Criterion in Variational Framework} \label{sec:dcvf}
With the denoising autoencoder criterion~\citep{Seung1998, Vincent2008}, the input is corrupted according
to some noise distribution, and the model needs to learn to reconstruct the
original input or maximize the log-probability of the clean input $\mathbf{x}$,
given the corrupted input $\mathbf{\tilde{x}}$.
Before applying the denoising criterion to the variational
autoencoder, we shall investigate a synthesized inference 
formulation of VAE in order to comprehend the consequences of the 
denoising criterion.

\begin{proposition}
    Let $q_\phi(\mathbf{z}|\mathbf{\tilde{x}})$ be a Gaussian distribution
    such that $q_\phi(\mathbf{z}|\mathbf{\tilde{x}}) = \mathcal{N}(\mathbf{z}|\mu_\phi(\mathbf{\tilde{x}}),\sigma_\phi(\mathbf{\tilde{x}}))$
    where $\mu_\phi(\mathbf{\tilde{x}})$ and $\sigma_\phi(\mathbf{\tilde{x}})$ are non-linear functions of $\mathbf{\tilde{x}}$.
    Let $p(\mathbf{\tilde{x}}|\mathbf{x})$ be a known corruption
    distribution around $\mathbf{x}$. Then, 
    \begin{equation}
        \mathbb{E}_{p(\mathbf{\tilde{x}}|\mathbf{x})} \left[ q_\phi(\mathbf{z}|\mathbf{\tilde{x}}) \right] 
        = \int_\mathbf{\tilde{x}} q_\phi(\mathbf{z}|\mathbf{\tilde{x}})p(\mathbf{\tilde{x}}|\mathbf{x})d\mathbf{\tilde{x}}
            \label{eqn:pro1}
    \end{equation}
    is a mixture of Gaussian.
\end{proposition}

Depending on whether the distribution is over a continuous or discrete variables,
the integral in Equation~\ref{eqn:pro1} can be replaced by a summation. It is instructive to consider the distribution over discrete domain to see 
that Equation~\ref{eqn:pro1} has a form of mixture of Gaussian - that is,
each time we sample $\mathbf{\tilde{x}} \sim p(\mathbf{\tilde{x}}|\mathbf{x})$ and 
substitute into $q(\mathbf{z}|\mathbf{\tilde{x}})$, we get different Gaussian distributions.

\begin{example}
    Let $\mathbf{x} \in \lbrace 0, 1\rbrace^D$ be a $D$-dimension observation, and consider a Bernoulli corruption distribution $p_{\bm{\pi}}(\mathbf{\tilde{x}}|\mathbf{x}) = Ber(\bm{\pi})$ around the input $\mathbf{x}$. Then,
    \begin{equation}
        \mathbb{E}_{p_{\bpi}(\mathbf{\tilde{x}}|\mathbf{x})} \left[ q_\phi(\mathbf{z}|\mathbf{\tilde{x}}) \right] 
        = \sum_{i=1}^{K} q_\phi(\mathbf{z}|\mathbf{\tilde{x}}_i)p_{\bm{\pi}}(\mathbf{\tilde{x}}_i|\mathbf{x})
    \end{equation}
    has the form of a finite mixture of Gaussian and the number of mixture component $K$ is $2^D$.
\end{example}

As mentioned in the previous section, usually a feedforward neural network is 
used for the inference network. In the case of the Bernoulli distribution 
as a corrupting distribution and $q_\phi(\mathbf{z}|\mathbf{\tilde{x}})$ is a Gaussian distribution,
we will have $2^D$ Gaussian mixture components and all of them share the parameter $\phi$.

\begin{example}
    Consider a Gaussian corruption model $p(\mathbf{\tilde{x}}|\mathbf{x}) = N(\mathbf{x}|\mathbf{0}, \sigma I)$.
    Let $q_\phi(\mathbf{z}|\mathbf{\tilde{x}})$ be a Gaussian inference network. 
    Then,
    \begin{equation}
        \mathbb{E}_{p(\mathbf{\tilde{x}}|\mathbf{x})} \left[ q_\phi(\mathbf{z}|\mathbf{\tilde{x}}) \right] 
        = \int_\mathbf{\tilde{x}} q_\phi(\mathbf{z}|\mathbf{\tilde{x}})p(\mathbf{\tilde{x}}|\mathbf{x}) d\mathbf{\tilde{x}}.
        \label{eqn:ex2_eqn}
    \end{equation}
    \begin{enumerate}
        \item If $q_\phi(\mathbf{z}|\phi\trns\mathbf{\tilde{x}}) = \mathcal{N}(\mathbf{z}|\mu=\phi\trns\mathbf{\tilde{x}},\sigma=\sigma^2I)$ 
    such that the mean parameter is a  linear model of weight vector $\phi$ and input $\mathbf{\tilde{x}}$,
    then the Equation~\ref{eqn:ex2_eqn} is a Gaussian distribution.\\
        \item If $q_\phi(\mathbf{z}|\mathbf{\tilde{x}}) = \mathcal{N}(\mathbf{z}|\mu(\mathbf{\tilde{x}}),\sigma(\mathbf{\tilde{x}}))$
    where $\mu(\mathbf{\tilde{x}})$ and $\sigma(\mathbf{\tilde{x}})$ are non-linear functions of $\mathbf{\tilde{x}}$,
    then the Equation~\ref{eqn:ex2_eqn} is an infinite mixture of Gaussian. 
    \end{enumerate}
\end{example}

In practice, there will be infinitely many number of Gaussian mixture components as in the second case, 
all of whose parameters are predicted by 
a single neural network. In other words, the inference neural network will learn which Gaussian distribution is needed for the given input $\btx$\footnote{Note that the mixture components are encoded in a vector form.}.


We can see this corruption procedure as adding a stochastic layer to the bottom of the inference 
network. For example, we can define a corruption network $p_{\bpi}(\btx|\bx)$ which is a neural
network where the input is $\bx$ and the output is stochastic units (e.g., Gaussian or Bernoulli 
distributions). Then, it is also possible to learn the parameter $\bpi$ of the corruption network 
by backpropagation using the re-parameterization trick. Note that a similar idea is explored in 
IWAE \citep{Burda2015}. However, our method is different in a sense that we use 
the denoising variational lower bound described below.

%
%

\subsection{The Denoising Variational Lower Bound}
\label{subsec:vidc}
Previously, we described that integrating the denoising criterion into 
the variational auto-encoding framework is equivalent to having a 
stochastic layer at the bottom of the inference network, and then 
estimating the variational lower bound becomes intractable because 
$\eE_{p(\btx|\bx)}[q_{\phi}(\bz|\btx)]$
requires integrating out 
the noise $\mathbf{\tilde{x}}$ for a corruption distribution. 
Before introducing the denoising variational lower bound, let us 
examine the variational lower bound when an additional stochastic layer is added to the inference network and integrate over the stochastic variables.
\begin{lemma}
Consider an approximate posterior distribution of the following form:
    \begin{displaymath}
        q_{\Phi}(\mathbf{z}|\mathbf{x}) = \int_{\mathbf{z}^\prime} q_{\varphi}(\mathbf{z}|\mathbf{z}^\prime)q_{\psi}(\mathbf{z}^\prime|\mathbf{x}) d\mathbf{z}^\prime,
    \end{displaymath}
    here, we use $\Phi = \{\varphi, \psi\}$. Then, given  $p_{\theta}( \mathbf{x}, \mathbf{z})
    = p_{\theta}(\mathbf{x}|\mathbf{z})p(\mathbf{z})$, we obtain the
    following inequality:
    \begin{displaymath}
        \log p_\theta(\mathbf{x}) \geq \mathbb{E}_{q_{\Phi}(\mathbf{z}|\mathbf{x})}
        \left[ \log \frac{p_\theta(\mathbf{x},\mathbf{z})}{q_{\varphi}(\mathbf{z}|\mathbf{z}^\prime)}\right] \geq 
        \mathbb{E}_{q_{\Phi}(\mathbf{z}|\mathbf{x})} \left[ \log \frac{p_\theta(\mathbf{x},\mathbf{z})}{q_{\Phi}(\mathbf{z}|\mathbf{x})}\right].
    \end{displaymath}
    \label{lemma1}
\end{lemma}
Refer to the Appendix for the proof. Note that $q_{\psi}(\mathbf{z}^\prime|\mathbf{x})$ can be either parametric or non-parametric distribution.
We can further show that this generalizes to multiple stochastic layers in the inference network.
\begin{theorem}
Consider an approximate posterior distribution of the following form 
    \begin{displaymath}
        q_{\Phi}(\mathbf{z}|\mathbf{x}) = \int_{\mathbf{z}^{1}\cdots \mathbf{z}^{L-1}}
        q_{\phi^L}(\mathbf{z}|\mathbf{z}^{L-1}) \cdots q_{\phi^1}(\mathbf{z}^1|\mathbf{x}) d\mathbf{z}^1\cdots d\mathbf{z}^{L-1}
    \end{displaymath}
    Then, given 
    $p_{\theta}( \mathbf{x}, \mathbf{z}) = p_{\theta}(\mathbf{x}|\mathbf{z})p(\mathbf{z})$, we obtain the following inequality:
    \begin{displaymath}
        \log p_\theta(\mathbf{x}) \geq \mathbb{E}_{q_\Phi(\mathbf{z}|\mathbf{x})}
        \left[ \log \frac{p_\theta(\mathbf{x},\mathbf{z})}{\prod^{L-1}_{i=1}q_{\phi^i}(\mathbf{z}^{i+1}|\mathbf{z}^i)}\right] \geq 
        \mathbb{E}_{q_\Phi(\mathbf{z}|\mathbf{x})} \left[ \log \frac{p_\theta(\mathbf{x},\mathbf{z})}{q_{\Phi}(\mathbf{z}|\mathbf{x})}\right],
    \end{displaymath}
    where $\mathbf{z} = \mathbf{z}^{L}$ and $\mathbf{x} = \mathbf{z}^{1}$.
    \label{theorem1}
\end{theorem}
The proof is presented in the Appendix. Theorem~\ref{theorem1} illustrates that adding more stochastic layers gives tighter lower bound. 

We now use Lemma~\ref{lemma1} to derive the \textit{denoising variational lower bound}. For the approximate distribution $\tq_{\phi}(\bz|\bx) = \int q_\phi(\bz|\btx)p(\btx|\bx) d\btx$, we can write the standard variational lower bound as follows:
\begin{align}
    \log p_\theta(\mathbf{x}) \geq E_{\tq_{\phi}(\mathbf{z}|\mathbf{x})} \left[ \log \frac{p_\theta(\mathbf{x},\mathbf{z})}
                {q_{\phi}(\mathbf{z}|\mathbf{\tilde{x}})}\right]
            = E_{\tq_{\phi}(\mathbf{z}|\mathbf{x})} \left[ \log \frac{p_\theta(\mathbf{x},\mathbf{z})}{\tq_{\phi}(\bz|\bx)}\right] \eqdef \mathcal{L}_{cvae}.
  	\label{eqn:cvae}
\end{align} 
Applying Lemma~\ref{lemma1} to Equation~\ref{eqn:cvae}, we can pull out the expectation in the denominator outside of the log
and obtain the denoising variational lower bound:
\begin{equation}
    \mathcal{L}_{dvae} \eqdef E_{\tq_{\phi}(\mathbf{z}|\mathbf{x})} 
    \left[ \log  \frac{ p_\theta (\mathbf{x},\mathbf{z})}{ q_\phi(\mathbf{z}|\mathbf{\tilde{x}})}\right].
    \label{eqn:Ldvae}
\end{equation}
Note that the $p_\ta(\bx,\bz)$ in the numerator of the above equation is a function of $\bx$ not $\btx$. That is, given corrupted input $\btx$ (in the denominator), the $\cL_{dvae}$ objective tries to reconstruct the original input $\bx$ not the corrupted input $\btx$. This \textit{denoising} criterion is different from the popular \textit{data augmentation} approach where the model tries to reconstruct the corrupted input.

By the Lemma~\ref{lemma1}, we finally have the following: 
\bea 
\log p_\theta(\mathbf{x}) \geq \mathcal{L}_{dvae} \geq \mathcal{L}_{cvae}.
\eea

It is important to note that the above does not necessarily mean  that $\mathcal{L}_{dvae} \geq \mathcal{L}_{vae}$ where $\mathcal{L}_{vae}$ is the lower bound
of VAE with Gaussian distribution in the inference network. This is because
$\tq_{\phi}(\mathbf{z}|\mathbf{x})$ in $\mathcal{L}_{cvae}$ depends on a corruption distribution while $q_{\phi}(\mathbf{z}|\mathbf{x})$ in $\mathcal{L}_{vae}$ does not. 

Note also that $\tq_{\phi}(\mathbf{z}|\mathbf{x})$ has the capacity to cover a much broader class of 
distributions than $q_{\phi}(\mathbf{z}|\mathbf{x})$. The distributions that $q_{\phi}(\mathbf{z}|\mathbf{x})$ can cover is an instance of the distributions that $\tq_{\phi}(\mathbf{z}|\mathbf{x})$ can cover.
This makes it possible for $\mathcal{L}_{dvae}$ to be a tighter lower bound of $\log p_\theta(\bx)$ than $\mathcal{L}_{vae}$.
For example, suppose that $p(\mathbf{z}|\mathbf{x})$ consists of
multiple modes.
Then, $\tq_{\phi}(\mathbf{z}|\mathbf{x})$ has the potential 
of modeling more than a single mode,
whereas it is impossible to model multiple modes of $p(\mathbf{z}|\mathbf{x})$ 
from $q_{\phi}(\mathbf{z}|\mathbf{x})$ regardless of which lower bound of $\log p_\theta(\bx)$ is used as the objective function.
However, be aware that it is also possible to make the $\cL_{cvae}$ a looser lower bound than $\cL_{vae}$ by choosing a very inefficient corruption distribution $p(\mathbf{\tilde{x}}|\mathbf{x})$ such that 
it completely distorts the input $\mathbf{x}$ in such a way to lose all useful information required for the reconstruction, resulting in $\cL_{cvae} < \cL_{vae}$. 
Therefore, for $\mathcal{L}_{dvae}$, it is important to choose a sensible corruption distribution.

A question that arises when we consider $\mathcal{L}_{dvae}$ is what is the underlying meaning of maximizing $\cL_{dvae}$. As mentioned earlier, 
the aim of variational objective is to minimize the distributions between approximate 
posterior and true posterior distribution, i.e.
$\KL_{dvae} = \KL\left(\tq_\phi(\mathbf{z}\|\mathbf{x})||p(\mathbf{z}|\mathbf{x})\right) = \log p_\theta(\mathbf{x}) - \mathcal{L}_{cvae}$.
However, $\mathcal{L}_{dvae}$ definitely does not minimizes only the KL between 
approximate posterior and true posterior distribution as we can observe that
$\KL_{cvae} = \KL\left(\tq_\phi(\mathbf{z}\|\mathbf{x})||p(\mathbf{z}|\mathbf{x})\right) \geq \log p_\theta(\mathbf{x}) - \mathcal{L}_{dvae}$.
This illustrates that $\KL_{dvae} \leq \KL_{dvae}$.
Nonetheless, $\KL_{dvae}$ provides tractable way to optimize from the approximate posterior distribution $q(\mathbf{z}|\mathbf{z})$.
Thus, it is interesting to see the following proposition.

\begin{proposition} Maximizing $\cL_{dvae}$ is equivalent to minimizing the following objective 
\bea
\eE_{p(\btx|\bx)}[\KL(\tq_\phi(\bz|\btx)|| p(\bz|\bx))].\label{eqn:exp_KL1}
\eea
\label{prop2}
In other words, $\log p_\theta(\bx) = \cL_{dvae} + \eE_{p(\btx|\bx)}[\KL(\tq_\phi(\bz|\btx)|| p(\bz|\bx))]$. 
\end{proposition}

The proof is presented in the Appendix.
Proposition~\ref{prop2} illustrates that maximizing $\mathcal{L}_{dvae}$ is equivalent to minimizing the expectation 
of the KL between the true posterior distribution and approximate posterior distribution \textit{over all noised inputs} from $p(\btx|\bx)$. We believe that this is indeed an effective objective because the inference network tries to learn to map not only the training data point but also its corrupted variations to the true posterior distribution, resulting in a more robust training of the inference network to unseen data points. As shown in Theorem~\ref{theorem1}, this argument also applies for 
multiple stochastic layers of inference network.

\subsection{Training Procedure}
\label{sec:training_procedure}


One may consider a simple way of training VAE with the denoising 
criterion, which is similar to how the vanilla denoising 
autoencoder is trained: (i) sample a corrupted input $
\mathbf{\tilde{x}}^{(m)} \sim p(\mathbf{\tilde{x}}|\mathbf{x})$, 
(ii) sample $\mathbf{z}^{(l)} \sim q(\mathbf{z}|\mathbf{\tilde{x}}
^{(m)})$, and (iii) sample reconstructed images from the generative 
network $p_\ta(\mathbf{x}|\mathbf{z}^{(l)})$. This procedure is 
very similar to the regular VAE except that the input is corrupted 
by a noise distribution at every update. 


The above procedure can be seen as a special case of optimizing the following objective which can be easily approximated by Monte Carlo sampling.
\begin{align}
&\cL_{dvae} = \mathbb{E}_{q(\bz|\mathbf{\tilde{x}})} \mathbb{E}_{p(\mathbf{\tilde{x}}|\mathbf{x})} \left[
        \log \frac{p_\theta(\mathbf{x},\mathbf{z})}{ q_\phi(\mathbf{z}|\mathbf{\tilde{x}}) } \right]\simeq 
\f{1}{MK}\sm{m}{M} \sm{k}{K}\log \f{p_\phi(\bx, \bz^{(k|m)})}{q_\phi(\bz^{(k|m)} | \btx^{(m)})}
    \label{eqn:training_procedure1}
\end{align}
where $\btx^{(m)} \sim p(\btx | \bx)$ and $\bz^{(k|m)} \sim q_\phi(\bz | \btx^{(m)})$.
In the experiment section, we call the estimator of Equation~\ref{eqn:training_procedure1} DVAE. Although in the above we applied the denoising criterion for VAE (resulting in DVAE) as a demonstration, but the proposed procedure is applicable to other variational methods using inference networks as well. For example, the training procedure for IWAE with denoising
criterion can be formulated with Monte Carlo approximation: 
\begin{align}
    &\cL_{diwae} = \mathbb{E}_{q(\bz|\mathbf{\tilde{x}})} \mathbb{E}_{p(\mathbf{\tilde{x}}|\mathbf{x})} \left[
\log \sm{m}{M} \sm{k}{K}\frac{p_\theta(\mathbf{x},\mathbf{z}^{(k|m)})}{ q_\phi(\mathbf{z}^{(k|m)}|\mathbf{\tilde{x}}^{(m)}) } \right]\simeq 
        \log \f{1}{MK}\sm{m}{M} \sm{k}{K} \f{p_\phi(\bx, \bz^{(k|m)})}{q_\phi(\bz^{(k|m)} | \btx^{(m)})} .
        \label{eqn:training_procedure2}
\end{align}
where $\btx^{(m)} \sim p(\btx | \bx)$, $\bz^{(k|m)} \sim q_\phi(\bz | \btx^{(m)})$, and Monte Carlo sample size is set to 1.
We named the following estimator of Equation~\ref{eqn:training_procedure2} as DIWAE.

\section{Experiments} 
\label{sec:exp}
We conducted empirical studies of DVAE under the denoising variational
lower bound as discussed in Section~\ref{sec:dcvf}. 
To assess whether adding a denoising criterion to the variational auto-encoding
models enhance the performance or not, we tested on the denoising criterion
on VAE and IWAE throughout the experiments. 
As mentioned in Section~\ref{subsec:vidc}, 
since the choice of the corruption distribution is crucial, we compare on different corruption distributions of various noise levels.

We consider two datasets, the binarized MNIST dataset and the Frey face dataset.
The MNIST dataset contains 60,000 images for training and 10,000 images for test and each of the images is 28 $\times$ 28 pixels for handwritten digits from 0 to 9 \citep{LeCun1998}.  
Out of the 60,000 training examples, we used 10,000 examples as validation set to tune the hyper-parameters of our model.
We use the binarized version of MNIST, where each pixel of an image is sampled from $\{0,1\}$
according to its pixel intensity value.
The Frey Face\footnote{Available at http://www.cs.nyu.edu/$\sim$roweis/data.html.} 
dataset consists of 2000 images of Brendan Frey's face. We split the images
into 1572 training data, 295 validation data, and 200 test data. We normalized the images such that each pixel value ranges between $[0,1]$.

Throughout the experiments, we used the same neural network architectures for VAE and IWAE. Also, a single stochastic layer with 50 latent variables is used for both VAE and IWAE. For the  generation network, we used a neural network of two hidden layers each of which has 200 units. For the inference network, we tested two architectures, one with a single hidden layer and the other with two hidden layers. We then used 200 hidden units for both of them.
We used softplus activations for VAE and tanh activations for IWAE following the same configuration of the original papers of \citet{Kingma2014vae} and \citet{Burda2015}.
For binarized MNIST dataset, the last layer of the generative network
was sigmoid and the usual cross-entropy term was used. For the Frey Face dataset where the input value is real numbers, we used Gaussian stochastic units for the output layer of the generation network.  

For all our results, we ran 10-fold experiments.
We optimized all our models with ADAM \citep{Kingma2015}. We set the batch size to 100 and the learning rate was selected from 
a discrete range chosen based on the validation set.
We used 1 and 5 samples of $\mathbf{z}$ per update for VAE and 5 samples for IWAE. Note that using 1 sample for IWAE is equivalent to VAE.
The reported results were only trained with training set, not including the validation set.
\begin{table}[t]
\center
\caption{Negative variational lower bounds using different corruption levels on MNIST (the lower, the better). 
The salt-and-pepper noises are injected to data $\mathbf{x}$ during the training.}
\label{exp:mnist}
\begin{tabular}{ | l | c | c | c | c | c |}\hline
    \multirow{2}*{Model} & \# Hidden & \multicolumn{4}{c|}{Noise Level} \\ \cline{3-6}
                         & Layers    & 0 & 5 & 10 & 15 \\ \hline\hline
    DVAE  (K=1)    & 1  &  96.14 $\pm$ 0.09 & \bf{95.52 $\pm$ 0.12*} & \bf{96.12 $\pm$ 0.06} & 96.83 $\pm$ 0.17 \\
    DVAE  (K=1)    & 2  &  95.90 $\pm$ 0.23 & \bf{95.34 $\pm$ 0.17*} & \bf{95.65 $\pm$ 0.14} & 96.17 $\pm$ 0.17 \\ \hline
    DVAE  (K=5)    & 1  &95.20 $\pm$ 0.07 & \bf{95.01 $\pm$ 0.04*} & 95.55 $\pm$ 0.07 & 96.41 $\pm$ 0.11 \\ 
    DVAE  (K=5)    & 2  &95.01 $\pm$ 0.07 & \bf{94.71 $\pm$ 0.13*} & \bf{94.90 $\pm$ 0.22} & 96.41 $\pm$ 0.11 \\ \hline
    DIWAE  (K=5)   & 1  &  94.36 $\pm$ 0.07 & \bf{93.67 $\pm$ 0.10*} & \bf{93.97 $\pm$ 0.07} & \bf{94.35 $\pm$ 0.08} \\
    DIWAE  (K=5)   & 2  &  94.31 $\pm$ 0.07 & \bf{93.08 $\pm$ 0.08*} & \bf{93.35 $\pm$ 0.13} & \bf{93.71 $\pm$ 0.07} \\ \hline
\end{tabular}
\vspace{-0.1cm}
\end{table}

Following common practices of choosing a noise distribution, we deployed the {\em salt and pepper} noise to the binary MNIST dataset and Gaussian noise to the real-valued Frey Face dataset. Table~\ref{exp:mnist} presents the negative variational lower bounds with
respect to different corruption levels on the MNIST dataset. 
Similarly, Table~\ref{exp:Frey Face} presents the negative variational lower bound using unnormalized generation networks, with respect to different corruption levels on the Frey Face dataset. Note that when the corruption level is set to zero, DVAE and DIWAE are identical to VAE and IWAE, respectively.
%
\begin{table}[b]
\vspace{-0.05cm}
\center
\caption{Negative variational lower bound using  different corruption levels on the Frey Face
dataset. Gaussian noises are injected to data $\mathbf{x}$ during the training.}
\label{exp:Frey Face}
\begin{tabular}{ | l | c | c | c | c | c |}\hline
    \multirow{2}*{Model} & \# Hid.    & \multicolumn{4}{c|}{Noise Level} \\ \cline{3-6}
                         & Layers       & 0 & 2.5 & 5 & 7.5 \\ \hline\hline
    DVAE (K=1) & 1 &  1304.79 $\pm$ 5.71 & \bf{1313.74 $\pm$ 3.64*} & \bf{1314.48 $\pm$ 5.85} & 1293.07 $\pm$ 5.03 \\ 
    DVAE (K=1) & 2 &  1317.53 $\pm$ 3.93 & \bf{1322.40 $\pm$ 3.11*} & \bf{1319.60 $\pm$ 3.30} & 1306.07 $\pm$ 3.35 \\ \hline
    DVAE (K=5) & 1 &  1306.45 $\pm$ 6.13 & \bf{1320.39 $\pm$ 4.17*} & \bf{1313.14 $\pm$ 5.80} & 1298.40 $\pm$ 4.74 \\ 
    DVAE (K=5) & 2 &  1317.51 $\pm$ 3.81 & \bf{1324.13 $\pm$ 2.62*} & \bf{1320.99 $\pm$ 3.49} & \bf{1317.56 $\pm$ 3.94} \\ \hline
    DIWAE (K=5)& 1 &  1318.04 $\pm$ 2.83 & \bf{1320.18 $\pm$ 3.43} & \bf{1333.44 $\pm$ 2.74*} & 1305.38 $\pm$ 2.97 \\
    DIWAE (K=5)& 2 &  1320.03 $\pm$ 1.67 & \bf{1334.77 $\pm$ 2.69*} & \bf{1323.97 $\pm$ 4.15} & 1309.30 $\pm$ 2.95 \\ \hline
\end{tabular}
\end{table}

In the following, we analyze the results by answering questions on the experiments.

{\em \textbf{Q}: Does adding the denoising criterion improve the performance of variational autoencoders?}

Yes. All of the methods with denoising criterion surpassed the performance of vanilla VAE and vanilla IWAE as shown in
Table~\ref{exp:mnist} and Table~\ref{exp:Frey Face}. But, it is dependent on the choice of proper corruption level; for a large amount of noise, as we expected, it tends to perform worse than the vanilla VAE and IWAE. \\

{\em \textbf{Q}: How sensitive is the model for the type and the level of the noise?}

\begin{wrapfigure}{r}{0.63\textwidth}
    \vspace{-0.7cm}
    \includegraphics[width=0.60\textwidth]{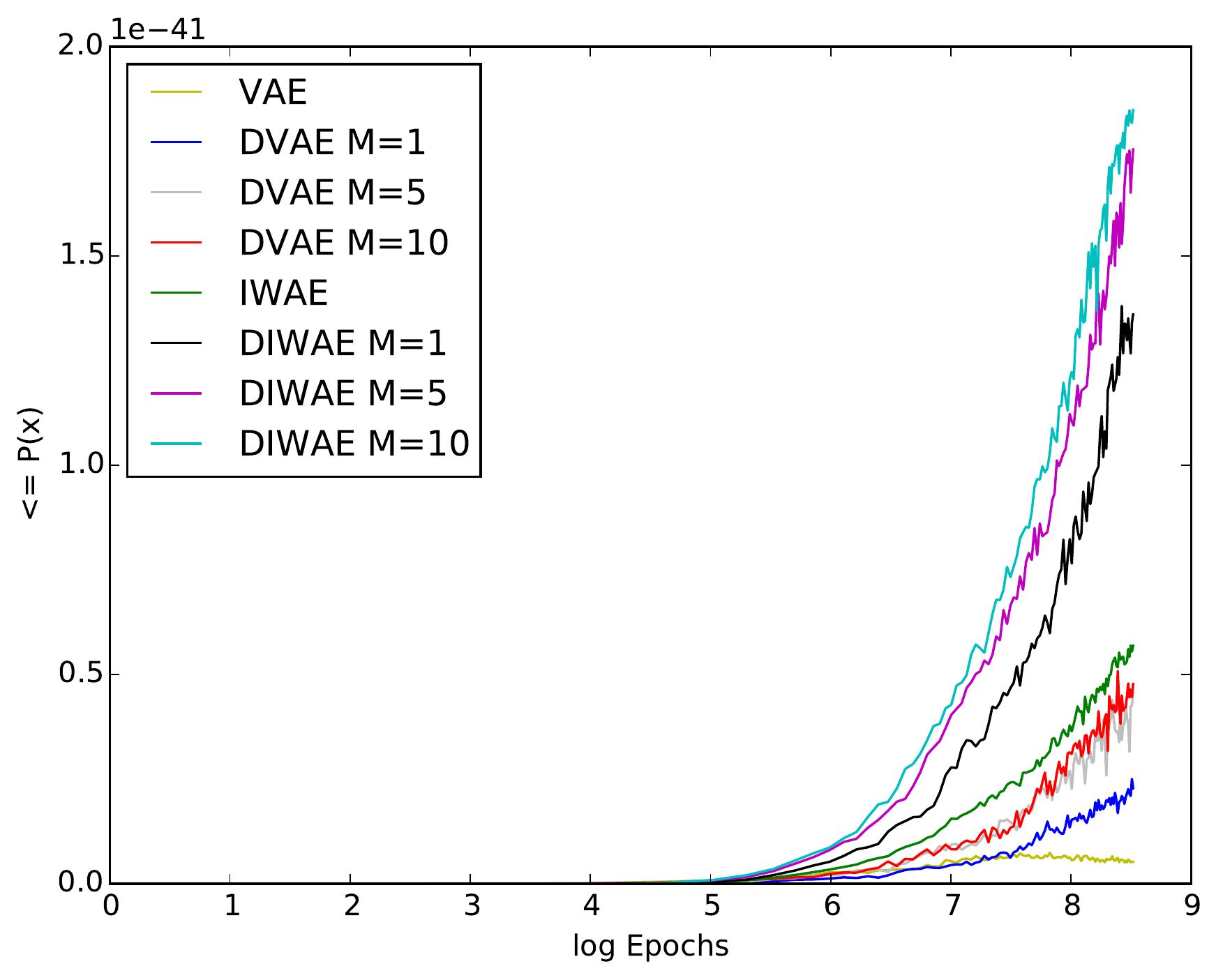}
    \vspace{-0.2cm}
    \caption{Denoising Variational Lower Bound for DVAE and DIWAE}
    \label{fig:expMLs}
\end{wrapfigure}
It seems that both of the models are not very sensitive with respect to the two types of noises: Gaussian and salt and pepper. They are more sensitive to the \textit{level} of the noise rather than the \textit{type}. Based on the experiments, the optimal corruption level lies in between $(0,5]$ since all of the results in that range are better than the one with $0$\% noise. It is natural to see this result considering that, when the noise level is excessive, (i) the model will lose information required to reconstruct the original input and that (ii) there will be large gap between the distributions of the (corrupted) training dataset and the test dataset.

{\em \textbf{Q}: How do the sample sizes $M$ affect to the result?}

In Figure~\ref{fig:expMLs}, we show the results on different configurations of $M$. 
As shown, increasing the sample size helps to converge faster in terms of the number of epochs 
and converge to better log-likelihood. The converged values of VAE are $94.97$, $94.44$, 
and $\bf{94.39}$ for $M=1,5,$ and $10$ respectively, and $93.17$, $92.89$, and $\bf{92.85}$ for IWAE.
Note, however, that increasing sample size requires more computation. 
Thus, in practice using $M=1$ seems a reasonable choice.

{\em \textbf{Q}: What happens when we replace the neural network in the inference network with some other types of model?}

Several applications have demonstrated that recurrent neural network can be more powerful
than neural netework. Here, we tried replacing neural network in the 
inference network with gated recurrent neural network that consist of single 
recurrent hidden layers with five time steps \citep{Chung2014}. 
We denote these models DVAE (GRU) and DIWAE (GRU) where GRU stands for gated recurrent
units.

Table~\ref{exp:gru} demonstrates the results with different noise level on MNIST dataset.
We notice that when VAE combined with GRU tend to severely overfit on the training data
and it actually performed worse than having a neural network at the inference network.
However, denoising criterion redeems the overfitting behaviour and produce much better
 results comparing with both VAE (GRU) and DVAE with regular neural networks.
Similarly, IWAE combined with GRU also showed overfitting behaviour although it gave
better results thatn DIWAE with neural networks. As well, DIWAE (GRU) gave the best 
performance among all models we experimented with.

\begin{table}[htp]
\center
\caption{Negative variational lower bounds using different corruption levels on MNIST (the lower, the better) 
    with recurrent neural network as a inference network. 
The salt-and-pepper noises are injected to data $\mathbf{x}$ during the training.}
\label{exp:gru}
\begin{tabular}{ | l | c | c | c | c | c |}\hline
    \multirow{2}*{Model} & \# Hidden & \multicolumn{4}{c|}{Noise Level} \\ \cline{3-6}
                         & Layers    & 0 & 5 & 10 & 15 \\ \hline\hline
    DVAE (GRU)    & 1 & 96.07 $\pm$ 0.17 & \bf{94.30 $\pm$ 0.09*} & \bf{94.32 $\pm$ 0.12} & \bf{94.88 $\pm$ 0.11} \\ 
    DIWAE (GRU)   & 1 & 93.94 $\pm$ 0.06 & \bf{93.13 $\pm$ 0.11} & \bf{92.84 $\pm$ 0.07*} & \bf{93.03 $\pm$ 0.04} \\\hline
\end{tabular}
\vspace{-0.1cm}
\end{table}

{\em \textbf{Q}: Data augmentation v.s. data corruption?}

Here, we consider specific data augmentation where our data lies in between 0 and 1, $\mathbf{x} \in (0,1)^D$
like MNIST. We consider a new binary data point $\mathbf{x}^\prime \in \lbrace 0, 1\rbrace^D$ 
where the previous data is treated as a probability of each pixel value turning 
on, i.e. $p(\mathbf{x}^\prime) = \mathbf{x}$. Then, we augment the data by sampling the data from $\mathbf{x}$
at every iteration. Although, this setting is not realistic, but we were curious whether the performance 
of this data augmentation compare to denoising criterion. The performance of such data augmentation on
MNIST gave $93.88 \pm 0.08$ and $92.51 \pm 0.07$ for VAE and IWAE. Comparing these negative log-likelihood with the
performance of DVAE and DIWAE, which were $94.32 \pm 0.37$ and $93.83 \pm 0.06$,
data augmentation VAE outperformed DVAE but data augmentation IWAE was worse than DIWAE.

{\em \textbf{Q}: Can we propose a more sensible noise distribution?}

For all the experiment results, we have used a simple corruption distribution using a global corruption rate (the parameter of the Bernoulli distribution or the variance of the Gaussian distribution) to all pixels 
in the images. To see if a more sensible corruption can lead to an improvement, we also tested another corruption distribution by obtaining a \textit{mean image}. Here, we obtained the mean image by averaging all training images and then used the pixel intensity of the mean image as the corruption rate so that each pixel has different corruption rate which  statistically encodes at some extent the pixel-wise noise from the entire dataset. However, we could not observe a noticeable improvement from this compared to the version with the global corruption rate, although we believed that this is a better way of designing the corrupting distribution. One interesting direction is to use a parameterized corruption distribution and learn the parameter. This will be advantageous because we can use our denoising variational lower bound which it is tighter than the classical variational lower bound on noisy inputs. We leave
this for the future work.



%
\section{Conclusions}

In this paper, we studied the denoising criterion for a general class of variational inference models where 
the approximate posterior distribution is conditioned on the input $\mathbf{x}$. 
The main result of our paper was to introduce the denoising variational lower 
bound which, provided a sensible corruption function, can be tighter than 
the standard variational lower bound on noisy inputs. We claimed that this training criterion makes it possible to learn more flexible and robust approximate posterior distributions such as the mixture of Gaussian than the standard training method without corruption. In the experiments, we empirically observed that the proposed method can consistently help to improve the performance 
for the variational autoencoder and the importance weighted autoencoder. Although we observed considerable improvements for our experiments with simple corruption distributions, how to obtain the sensible corruption distribution is still an important open question. We think that learning with a parametrized corruption distribution or obtaining a better heuristic procedure will be important for the method to be applied more broadly.

{\bf Acknowledgments}: We thank the developers of
Theano~\citep{bergstra+al:2010-scipy-small} for
their great work. We thank NSERC, Compute Canada, CIFAR and Samsung for their support.

\newpage
\bibliography{dcvaf}
\bibliographystyle{iclr2016_conference}
\newpage

\setcounter{theorem}{0} 
\setcounter{lemma}{-1} 
\setcounter{proposition}{0} 

\section*{appendix}
\subsection*{Denoising Criterion in Variational Framework} 

\begin{lemma}
    For all nonnegative measurable functions $f,g:\mathbb{R}[0,\infty)$
        that satisfies $\int^{\infty}_{-\infty} f(\mathbf{x}) d\mathbf{x} = 1$,
    \begin{displaymath}
        \int^{\infty}_{-\infty} f(\mathbf{x})\log g(\mathbf{x}) d\mathbf{x} 
        \leq \int^{\infty}_{-\infty} f(\mathbf{x})\log f(\mathbf{x}) d\mathbf{x}
    \end{displaymath}
    \label{lemma0}
\end{lemma}
\begin{proof}
    Let $X$ be a random variable with probability density function $f(\mathbf{x})$.
    Consider the random variable $\log \left[ \frac{f(\mathbf{x})}{g(\mathbf{x})}\right]$.
    Consider $\mathbb{E}_{f(\mathbf{x})}\left[\log  \frac{g(\mathbf{x})}{f(\mathbf{x})}\right]$, which is
        $- \mathbb{E}_{f(\mathbf{x})}\left[\log  \frac{f(\mathbf{x})}{g(\mathbf{x})}\right]$,
    then, by Jensen's inequality, we have 
    \begin{displaymath}
        \mathbb{E}_{f(\mathbf{x})}\left[ \log \frac{g(\mathbf{x})}{f(\mathbf{x})}\right] 
            \leq \log \mathbb{E}\left[ f(\mathbf{x})\frac{g(\mathbf{x})}{f(\mathbf{x})}\right] 
            = \log \left(\int^{\infty}_{-\infty} g(\mathbf{x}) d\mathbf{x}\right) = 0
    \end{displaymath}   
    Thus, we showed that  $\mathbb{E}_{f(\mathbf{x})}\left[ \log g(\mathbf{x})\right] \leq 
                \mathbb{E}_{f(\mathbf{x})}\left[ \log f(\mathbf{x})\right].$
\end{proof}

\begin{lemma}
    Given a directed latent variable model that factorizes to 
    $p_{\theta}( \mathbf{x}, \mathbf{z}) = p_{\theta}(\mathbf{x}|\mathbf{z})p(\mathbf{z})$,
    consider an approximate posterior distribution that takes the form of  
    \begin{displaymath}
        q_\Phi(\mathbf{z}|\mathbf{x}) = \int_{\mathbf{z}^\prime} q_{\varphi}(\mathbf{z}|\mathbf{z}^\prime)q_{\psi}(\mathbf{z}^\prime|\mathbf{x}) d\mathbf{z}^\prime
    \end{displaymath}
    Then, we obtain the following inequality: 
    \begin{displaymath}
        \log p_\theta(\mathbf{x}) \geq \mathbb{E}_{q_{\Phi}(\mathbf{z}|\mathbf{x})}
        \left[ \log \frac{p_\theta(\mathbf{x},\mathbf{z})}{q_{\varphi}(\mathbf{z}|\mathbf{z}^\prime)}\right] \geq 
        \mathbb{E}_{q_{\Phi}(\mathbf{z}|\mathbf{x})} \left[ \log \frac{p_\theta(\mathbf{x},\mathbf{z})}{q_{\Phi}(\mathbf{z}|\mathbf{x})}\right].
    \end{displaymath}
\end{lemma}
\begin{proof}
    Suppose that we have the following conditions:
    A directed latent variable model that factorizes to 
    $p_{\theta}( \mathbf{x}, \mathbf{z}) = p_{\theta}(\mathbf{x}|\mathbf{z})p(\mathbf{z})$,
    consider an approximate posterior distribution that takes the form of  
    \begin{displaymath}
        q_\Phi(\mathbf{z}|\mathbf{x}) = \int_{\mathbf{z}^\prime} q_{\varphi}(\mathbf{z}|\mathbf{z}^\prime)q_{\psi}(\mathbf{z}^\prime|\mathbf{x}) d\mathbf{z}^\prime
    \end{displaymath}

    By Jensen's inequality, we can see that this is also a lower bound of the marginal log-likelihood:
    \begin{align}
        \mathbb{E}_{q_{\psi}(\mathbf{z}^\prime|\mathbf{x})} \mathbb{E}_{q_{\varphi}(\mathbf{z}|\mathbf{z}^\prime)}\left[ 
                \log \frac{p_\theta(\mathbf{x},\mathbf{z})}{ q_\varphi(\mathbf{z}|\mathbf{z}^\prime) } \right] 
        \leq    
        \log \mathbb{E}_{q_{\psi}(\mathbf{z}^\prime|\mathbf{x})} \mathbb{E}_{q_{\varphi}(\mathbf{z}|\mathbf{z}^\prime)}
                \left[ \frac{p_\theta(\mathbf{x},\mathbf{z})}{ q_\varphi(\mathbf{z}|\mathbf{z}^\prime) } \right] 
                = \log p_\theta(\mathbf{x})\nonumber.
    \end{align}
    We showed that the left inequality holds, 
    and now, we show that the right inequality is also satisfied. 
    \begin{align}
        \mathbb{E}_{q_\Phi(\mathbf{z}|\mathbf{x})} \left[ \log \frac{p_\theta(\mathbf{x},\mathbf{z})}{q_{\varphi}(\mathbf{z}|\mathbf{z}^\prime)}\right]
        &=  \mathbb{E}_{q_\Phi(\mathbf{z}|\mathbf{x})} \left[ \log p_\theta(\mathbf{x},\mathbf{z}) \right]- 
            \mathbb{E}_{q_\Phi(\mathbf{z}|\mathbf{x})} \left[ \log q_{\Phi}(\mathbf{z}|\mathbf{z}^\prime)\right]\nonumber\\
            \intertext{By applying Lemma~\ref{lemma0} to the second term, we get}
        &\geq  \mathbb{E}_{q_\Phi(\mathbf{z}|\mathbf{x})} \left[ \log p_\theta(\mathbf{x},\mathbf{z}) \right]- 
            \mathbb{E}_{q_\Phi(\mathbf{z}|\mathbf{x})} \left[ \log q_{\Phi}(\mathbf{z}|\mathbf{x})\right]\nonumber\\
        &=  \mathbb{E}_{q_\Phi(\mathbf{z}|\mathbf{x})} \left[ \log \frac{p_\theta(\mathbf{x},\mathbf{z})}{q_{\Phi}(\mathbf{z}|\mathbf{x})}\right].\nonumber
    \end{align}
    Thus, we get
    \begin{displaymath}
        \log p_\theta(\mathbf{x}) \geq \mathbb{E}_{q_{\Phi}(\mathbf{z}|\mathbf{x})}
        \left[ \log \frac{p_\theta(\mathbf{x},\mathbf{z})}{q_{\varphi}(\mathbf{z}|\mathbf{z}^\prime)}\right] \geq 
        \mathbb{E}_{q_{\Phi}(\mathbf{z}|\mathbf{x})} \left[ \log \frac{p_\theta(\mathbf{x},\mathbf{z})}{q_{\Phi}(\mathbf{z}|\mathbf{x})}\right].
    \end{displaymath}
\end{proof}
\begin{theorem}
    Given a directed latent variable model that factorizes to 
    $p_{\theta}( \mathbf{x}, \mathbf{z}) = p_{\theta}(\mathbf{x}|\mathbf{z})p(\mathbf{z})$,
    consider an approximate posterior distribution that takes the form of  
    \begin{displaymath}
        q_\phi(\mathbf{z}|\mathbf{x}) = \int_{\mathbf{z}^{1}\cdots \mathbf{z}^{L-1}}
        q_{\phi^L}(\mathbf{z}|\mathbf{z}^{L-1}) \cdots q_{\phi^1}(\mathbf{z}^1|\mathbf{x}) d\mathbf{z}^1\cdots d\mathbf{z}^{L-1}
    \end{displaymath}
    Then, we obtain the following inequality:
    \begin{displaymath}
        \log p_\theta(\mathbf{x}) \geq \mathbb{E}_{q_\Phi(\mathbf{z}|\mathbf{x})}
        \left[ \log \frac{p_\theta(\mathbf{x},\mathbf{z})}{\prod^{L-1}_{i=1}q_{\phi^i}(\mathbf{z}^{i+1}|\mathbf{z}^i)}\right] \geq 
        \mathbb{E}_{q_\Phi(\mathbf{z}|\mathbf{x})} \left[ \log \frac{p_\theta(\mathbf{x},\mathbf{z})}{q_{\Phi}(\mathbf{z}|\mathbf{x})}\right],
    \end{displaymath}
    where $\mathbf{z} = \mathbf{z}^{L}$ and $\mathbf{x} = \mathbf{z}^{1}$ .
\end{theorem}
\begin{proof}
    The sketch of the proof is basically applying the Lemma~\ref{lemma1} consecutively $L$ many times to the classical variational lower bound.
    \begin{align}
        \mathbb{E}_{q_\phi(\mathbf{z}|\mathbf{x})} &\left[ \log \frac{p_\theta(\mathbf{x},\mathbf{z})}{q_{\phi}(\mathbf{z}|\mathbf{x})}\right]
        =  \mathbb{E}_{q_\phi(\mathbf{z}|\mathbf{x})} \left[ \log \frac{p_\theta(\mathbf{x},\mathbf{z})}
        {\mathbb{E}_{q_{\phi^1}(\mathbf{z}^1|\mathbf{x})} \cdots \mathbb{E}_{q_{\phi^L}(\mathbf{z}|\mathbf{z}^{L-1})}               
    \left[ \prod^{L-1}_{i=1}q_{\phi^i}(\mathbf{z}^{i+1}|\mathbf{z}^i)\right]}\right]\nonumber
        \intertext{Applying Lemma~\ref{lemma1} to layer 1, we get}
                & \leq  \mathbb{E}_{q_\phi(\mathbf{z}|\mathbf{x})} 
                \left[ \log \frac{p_\theta(\mathbf{x},\mathbf{z})}
                    {q_{\phi^1}(\mathbf{z}^2|\mathbf{z}^1) \mathbb{E}_{q_{\phi^2}(\mathbf{z}^2|\mathbf{z}^1)} \cdots \mathbb{E}_{q_{\phi^L}(\mathbf{z}|\mathbf{z}^{L-1})}               
                \left[ \prod^{L-1}_{i=2}q_{\phi^i}(\mathbf{z}^{i+1}|\mathbf{z}^i)\right]}\right] \nonumber\\
        \intertext{ Applying Lemma~\ref{lemma1} to layer 2, we get}
                & \leq  \mathbb{E}_{q_\phi(\mathbf{z}|\mathbf{x})} 
                \left[ \log \frac{p_\theta(\mathbf{x},\mathbf{z})}
                    {q_{\phi^2}(\mathbf{z}^3|\mathbf{z}^2)q_{\phi^1}(\mathbf{z}^2|\mathbf{z}^1) \mathbb{E}_{q_{\phi^3}(\mathbf{z}^3|\mathbf{z}^2)} \cdots \mathbb{E}_{q_{\phi^L}(\mathbf{z}|\mathbf{z}^{L-1})}               
            \left[ \prod^{L-1}_{i=3}q_{\phi^i}(\mathbf{z}^{i+1}|\mathbf{z}^i)\right]}\right] \nonumber
        \intertext{ Repeatly applying Lemma~\ref{lemma1} to layers from 3 to $L$ , we get}
            &\leq \mathbb{E}_{q_\phi(\mathbf{z}|\mathbf{x})} 
        \left[ \log \frac{p_\theta(\mathbf{x},\mathbf{z})}{\prod^{L-1}_{i=1}q_{\phi^i}(\mathbf{z}^{i+1}|\mathbf{z}^i)}\right]\nonumber\\
            &\leq \log p_\theta(\mathbf{x})\nonumber
    \end{align}
    Therefore, the inequalities in Theorem~\ref{theorem1} are satisfied. 
\end{proof}

\begin{proposition}
    \begin{align}
        \argmax_{\phi,\theta} \mathcal{L}_{dvae} &\equiv \argmin_{\phi,\theta} 
    \left[\KL\left(\tq_\phi(\mathbf{z}|\mathbf{x})\|p(\mathbf{z}|\mathbf{x})\right) - \KL\left(\tq_\phi(\mathbf{z}|\mathbf{x}) \| q_\phi(\mathbf{z}|\mathbf{\tilde{x}})\right) \right]\nonumber\\
    &\equiv \argmin_{\phi,\theta}\mathbb{E}_{p(\mathbf{\tilde{x}}|\mathbf{x})} \left[ \KL\left( q_\phi(\mathbf{z}|\mathbf{\tilde{x}})\|p(\mathbf{z}|\mathbf{x})\right)\right]\nonumber
    \end{align}
\end{proposition}
\begin{proof}
    Let us consider $\theta$ beging fixed just for the sake of simpler analysis.
    \begin{align}
        \log p_\theta(\mathbf{x}) - \mathcal{L}_{dvae}
        &= \log p_\theta(\mathbf{x}) -\mathbb{E}_{\tq_\phi(\mathbf{z}|\mathbf{x})} 
         \left[ \log\frac{p_\theta(\mathbf{x},\mathbf{z})}{q_\phi(\mathbf{z}|\mathbf{\tilde{x}})} \right]\nonumber\\
        &= \mathbb{E}_{\tq_\phi(\mathbf{z}|\mathbf{x})} 
         \left[ \log \frac{q_\phi(\mathbf{z}|\mathbf{\tilde{x}})}{\tq_\phi(\mathbf{z}|\mathbf{x})} \right]
             + \KL\left(\tq_\phi(\mathbf{z}|\mathbf{x})\|p(\mathbf{z}|\mathbf{x})\right)\nonumber\\
        &= \KL\left(\tq_\phi(\mathbf{z}|\mathbf{x})\|p(\mathbf{z}|\mathbf{x})\right) - \KL\left(\tq_\phi(\mathbf{z}|\mathbf{x})\|q(\mathbf{z}|\mathbf{\tilde{x}})\right)\\
             \intertext{Thus we achieved the first equality.}
        &= \mathbb{E}_{\tq_\phi(\mathbf{z}|\mathbf{x})} 
          \left[ \log \frac{q_\phi(\mathbf{z}|\mathbf{\tilde{x}})}{\tq_\phi(\mathbf{z}|\mathbf{x})} \right]
         + \mathbb{E}_{\tq_\phi(\mathbf{z}|\mathbf{x})} \left[ \log \frac{\tq_\phi(\mathbf{z}|\mathbf{x})}{p(\mathbf{z}|\mathbf{x})}\right]\nonumber\\
        &= \mathbb{E}_{\tq_\phi(\mathbf{z}|\mathbf{x})} 
          \left[ \log \frac{q_\phi(\mathbf{z}|\mathbf{\tilde{x}})\tq_\phi(\mathbf{z}|\mathbf{x})}{\tq_\phi(\mathbf{z}|\mathbf{x})p(\mathbf{z}|\mathbf{x})} \right]\nonumber\\
        &= \mathbb{E}_{\tq_\phi(\mathbf{z}|\mathbf{x})} 
          \left[ \log \frac{q_\phi(\mathbf{z}|\mathbf{\tilde{x}})}{p(\mathbf{z}|\mathbf{x})} \right]\nonumber\\
        &= \mathbb{E}_{\tq_\phi(\mathbf{\tilde{x}}|\mathbf{x})} \mathbb{E}_{q_\phi(\mathbf{z}|\mathbf{\tilde{x}})} 
            \left[ \log \frac{q_\phi(\mathbf{z}|\mathbf{\tilde{x}})}{p(\mathbf{z}|\mathbf{x})} \right]\\
            &= \mathbb{E}_{p(\mathbf{\tilde{x}}|\mathbf{x})} \big[ \KL\left(q_\phi(\mathbf{z}|\mathbf{\tilde{x}})\|p(\mathbf{z}|\mathbf{x})\right) \big]
    \end{align}
    Thus, we achieved the second equality. Overall, maximizing $\mathcal{L}_{dvae}$ is equivalent to minimizing the expectation of KL between 
    the true posterior distribution and approximate posterior distribution for each noised input.
\end{proof}

\end{document}